\documentclass[10pt]{article} 
\usepackage[preprint]{tmlr}

\usepackage{amsmath,amsfonts,bm}

\def\eqref#1{equation~\ref{#1}}

\def\1{\bm{1}}

\DeclareMathAlphabet{\mathsfit}{\encodingdefault}{\sfdefault}{m}{sl}
\SetMathAlphabet{\mathsfit}{bold}{\encodingdefault}{\sfdefault}{bx}{n}

\usepackage{hyperref}
\usepackage{booktabs}
\usepackage{url}
\usepackage{algorithmicx}
\usepackage{algpseudocode}
\usepackage{graphicx}
\usepackage{amsmath,amssymb}
\usepackage{multirow}
\usepackage{algorithm}

\usepackage{amsthm}

\newtheorem{theorem}{Theorem}
\newtheorem{lemma}{Lemma}

\title{Chance-Constrained Inference for Hallucination Risk Control in Large Language Models\texorpdfstring{\thanks{This manuscript is currently under review.}}{}}

\author{Sreenivasan Mohandas \\
International Institute of Information Technology, Hyderabad, India \\
\texttt{sreenivasanm6@gmail.com}
}

\def\month{MM}  
\def\year{YYYY} 
\def\openreview{\url{https://openreview.net/forum?id=XXXX}} 

\begin{document}

\maketitle
\begin{abstract}
Large language models generate outputs stochastically and may produce fluent but invalid responses, including factual hallucinations. Existing mitigation strategies reduce average error rates but do not provide explicit control over the \emph{frequency} of such failures under repeated use. We formulate inference as a deployment-time risk control problem and introduce \emph{chance-constrained inference}, which directly bounds the probability of hallucinations among accepted generations. Hallucinations are modeled as stochastic constraint violations, and we show that confidence-based selective prediction does not, in general, imply probabilistic risk guarantees. To enforce chance constraints efficiently, we propose a sequential, anytime-valid inference procedure that adaptively certifies feasibility or infeasibility using finite samples, avoiding conservative fixed-sample bounds. Experiments on questions inspired by NaturalQuestions and controlled multi-hop question answering demonstrate reliable risk control, early detection of intrinsically infeasible inputs, and safe composition under repeated use, while confidence-based baselines fail to provide consistent guarantees.

\end{abstract}

\section{Introduction}

Large language models (LLMs) are increasingly deployed in applications such as question answering, decision support, and multi-step reasoning. Despite strong average performance, LLMs remain unreliable in a critical respect: they may generate fluent but invalid outputs, including factual hallucinations and constraint violations \citep{maynez2020faithfulness,zhang2023hallucination}. In many applications, even infrequent failures are unacceptable, particularly when models are queried repeatedly or embedded in agentic systems.

Most existing mitigation strategies aim to reduce average error rates. Retrieval-augmented generation \citep{lewis2020retrieval}, alignment via human feedback, and post-hoc hallucination detection \citep{farquhar2024semantic} improve empirical accuracy but do not provide explicit control over how often failures occur under stochastic decoding. A closely related line of work is selective prediction and abstention \citep{el-yaniv2010selective,geifman2017selective}, which rejects low-confidence outputs based on uncertainty estimates. While effective in supervised classification, confidence-based rejection does not generally provide probabilistic guarantees for generative models, where decoding randomness induces variability even for fixed inputs.

A key limitation of prior approaches is that hallucination is implicitly treated as a deterministic failure to be eliminated. In practice, modern decoding procedures such as temperature or nucleus sampling make generation inherently stochastic \citep{holtzman2020nucleus}, rendering hallucination a random event whose probability depends on the input, model, and decoding strategy. As a result, reliability must be defined and enforced at the level of the \emph{generation distribution}, rather than on individual outputs.

In contrast, decision making under uncertainty in fields such as operations research and control is commonly framed using \emph{chance constraints}, which require constraints to hold with high probability rather than almost surely \citep{prekopa1995stochastic,calafiore2006scenario}. Chance constraints provide interpretable and adjustable guarantees on failure frequency while avoiding overly conservative behavior.

In this work, we bring chance-constrained optimization to large language model inference. We formulate hallucination control as a deployment-time risk management problem and introduce \emph{chance-constrained inference} (CCI), which bounds the frequency of hallucinations among accepted generations under stochastic decoding. The framework operates entirely at inference time, without retraining, fine-tuning, or calibration of the underlying model, and provides explicit distribution-level reliability guarantees rather than heuristic confidence thresholds.

Enforcing probabilistic guarantees raises a practical challenge: fixed-sample concentration bounds can require many generations to certify small risk budgets, which is infeasible in latency-sensitive settings. We therefore adopt a \emph{sequential, anytime-valid} inference perspective, where samples are generated adaptively and feasibility or infeasibility is certified as soon as sufficient evidence is available.

Our main contributions are:
(i) a formulation of hallucination control as a chance-constrained inference problem for stochastic language generation;
(ii) a sequential, anytime-valid inference procedure that certifies feasibility or infeasibility using finite samples;
and (iii) an empirical evaluation demonstrating reliable risk control, early detection of intrinsically infeasible inputs, and safe composition under repeated use, in contrast to confidence-based baselines.

\section{Related Work}

Hallucination has been widely studied as a core reliability issue in large
language models, referring to fluent but incorrect or unsupported generations
\citep{maynez2020faithfulness,zhang2023hallucination}.
Prior work attributes hallucinations to data sparsity, model uncertainty, and
misalignment between training objectives and downstream tasks.
A large body of research focuses on post-hoc detection using entailment checks,
self-consistency, or semantic verification
\citep{farquhar2024semantic}.
While effective for auditing and analysis, these approaches do not provide
explicit control over the \emph{frequency} of hallucinations under stochastic
decoding or repeated use.

Selective prediction and abstention aim to improve reliability by rejecting
low-confidence outputs \citep{el-yaniv2010selective,geifman2017selective}.
In supervised classification, selective prediction admits formal guarantees
because predictions are deterministic functions of inputs.
In generative models, this idea is commonly instantiated via confidence-based
thresholding (Conf-SP), which rejects outputs with low heuristic confidence
scores.
Related heuristics such as self-consistency–based selective prediction (SC-SP),
which accepts outputs that are stable across multiple stochastic generations,
have been shown to reduce average hallucination rates
\citep{wang2023selfconsistency}.
However, neither confidence-based nor self-consistency–based selective
prediction provides explicit guarantees on the conditional frequency of
hallucinations under stochastic decoding or repeated use.

Conformal prediction provides distribution-free uncertainty quantification
with finite-sample guarantees.
Selective conformal prediction (SCP) adapts conformal calibration to accept or
reject predictions based on fixed-sample estimates, yielding marginal risk
control under i.i.d.\ assumptions \citep{angelopoulos2021gentle}.
Split conformal risk control (SCRC) further enforces global risk constraints by
calibrating acceptance thresholds on held-out data
\citep{bates2023calibration}.
While these methods provide important statistical guarantees, they rely on
fixed calibration budgets, offer only marginal (input-averaged) risk control,
and do not support adaptive stopping or early infeasibility detection.

Retrieval-augmented generation conditions models on external documents to improve
factual accuracy \citep{lewis2020retrieval}.
Although grounding reduces average hallucination rates, models may still produce
unsupported claims even when relevant evidence is available
\citep{shuster2021retrieval}.
Moreover, retrieval does not provide a formal mechanism to bound the probability
of incorrect outputs under stochastic decoding.
Chance-constrained inference is complementary: retrieval enriches the input,
while chance constraints explicitly regulate output risk at inference time.

Constrained and structured decoding enforces syntactic or structural validity
during generation, such as grammar- or schema-constrained decoding
\citep{lu2021neural}.
These techniques provide deterministic guarantees for format correctness but are
limited to constraints that can be symbolically encoded.
They do not address semantic correctness or factual validity, which are
inherently probabilistic and context-dependent.
Our framework generalizes constrained decoding by allowing probabilistic
enforcement of semantic constraints.

A growing literature studies uncertainty estimation and calibration for language
models using self-consistency, ensembles, or confidence prompting
\citep{wang2023selfconsistency,kadavath2022language}, as well as post-hoc
calibration techniques \citep{guo2017calibration,huang2024calibration}.
While uncertainty estimates are useful auxiliary signals, they do not by
themselves enforce reliability guarantees.
A model may be well-calibrated yet still violate constraints at an unacceptable
frequency when deployed repeatedly.
Our approach leverages uncertainty signals when available but enforces explicit
probabilistic bounds on violation rates.

Chance-constrained optimization is a classical framework for decision making
under uncertainty, requiring constraints to hold with high probability rather
than almost surely \citep{prekopa1995stochastic,calafiore2006scenario}.
Scenario-based methods provide finite-sample feasibility guarantees
\citep{campi2018introduction}.
We adapt these ideas to language model inference, where decisions correspond to
accepting or rejecting stochastic generations rather than selecting
deterministic actions.

Sequential hypothesis testing and anytime-valid concentration inequalities
enable guarantees under adaptive stopping
\citep{howard2021time,howard2021confidence}.
These techniques are widely used in online experimentation and safety
monitoring but have not been systematically applied to inference-time control
of generative models.
Our work integrates anytime-valid risk estimation with chance-constrained
inference, enabling practical deployment-time guarantees with adaptive sampling.

Recent systems work proposes end-to-end frameworks for hallucination mitigation
in deployed LLMs by combining uncertainty signals, semantic consistency checks,
and continuous monitoring \citep{pesaranghader2025hallucination}.
While operationally valuable, these approaches do not impose explicit
probabilistic constraints on the long-run frequency of hallucinations under
stochastic decoding, nor do they provide finite-sample guarantees under repeated
or sequential use.
Our work is complementary, focusing on inference-time risk control with formal
chance-constrained guarantees on conditional violation rates.

In summary, prior work reduces hallucinations on average, detects failures
post hoc, or rejects individual outputs based on confidence, consistency, or
structure.
However, none provide a unified mechanism for controlling the
\emph{distribution-level frequency} of semantic violations under stochastic
decoding and repeated use.
This work addresses that gap by introducing chance-constrained inference with
sequential, anytime-valid guarantees.

\section{Notation and Problem Setup}

We model large language model (LLM) inference as a stochastic generation process.
Let $\mathcal{X}$ denote the space of inputs, where each $x \in \mathcal{X}$ may consist
of a user query optionally augmented with retrieved context, and let $\mathcal{Y}$ denote
the space of possible outputs. A pretrained language model with parameters $\theta$ induces a conditional distribution
\begin{equation}
y \sim p_\theta(\cdot \mid x),
\end{equation}
where randomness arises from inference-time decoding procedures such as temperature or
nucleus sampling. For conceptual clarity, we may equivalently represent generation as
\begin{equation}
y = f_\theta(x,\omega),
\end{equation}
where $\omega$ is an auxiliary random variable capturing all decoding randomness.
All probabilistic statements are taken with respect to $\omega$, conditioned on a fixed
input $x$. 

For each input $x$, we define a set of valid outputs
$\mathcal{Y}_{\mathrm{safe}}(x) \subseteq \mathcal{Y}$.
Outputs outside this set are considered invalid.
Validity is encoded via a binary violation indicator
\begin{equation}
H(x,y) = \mathbb{I}\!\left[y \notin \mathcal{Y}_{\mathrm{safe}}(x)\right],
\end{equation}
where $H(x,y)=1$ corresponds to a hallucination or other constraint violation. Because generation is stochastic, $H(x,y)$ is itself a random variable.
We define the violation probability for input $x$ as
\begin{equation}
R(x) = \mathbb{P}_{y \sim p_\theta(\cdot \mid x)}\!\left[H(x,y)=1\right],
\end{equation}
which captures the intrinsic risk induced by the model and decoding procedure for that input. Separately, we define a utility function
\begin{equation}
U : \mathcal{X} \times \mathcal{Y} \rightarrow \mathbb{R},
\end{equation}
which measures task-specific usefulness, such as correctness or informativeness.
Utility and validity are treated as distinct notions: an output may be useful yet invalid,
or valid but low-utility. In deployment, not all generated outputs are returned to the user.
We model this behavior via an acceptance policy
\begin{equation}
A : \mathcal{X} \times \mathcal{Y} \rightarrow \{0,1\},
\end{equation}
where $A(x,y)=1$ indicates acceptance and $A(x,y)=0$ indicates rejection or abstention.
Acceptance may depend on inference-time signals such as external verifiers, confidence
estimates, or rule-based checks.

Finally, each input $x$ is associated with a risk budget $\epsilon(x)\in[0,1]$,
specifying the maximum tolerable frequency of violations among user-visible outputs
under repeated stochastic generation.
Our objective is to design inference-time procedures that maximize expected utility
while ensuring that violation frequency does not exceed $\epsilon(x)$. This distribution-level notion of reliability is essential in repeated-use and agentic
settings, where even small per-query risks may accumulate over time.

\section{Constraint Modeling and Hallucination}

We model hallucinations and other failure modes as violations of explicit constraints
imposed on generated outputs.
This abstraction enables domain-agnostic reasoning about reliability and allows
probabilistic guarantees to be enforced at the level of the generation distribution,
rather than on individual samples. Rather than assuming a single notion of correctness, we consider a collection of
binary constraint functions
\begin{equation}
H_k : \mathcal{X} \times \mathcal{Y} \rightarrow \{0,1\}, \quad k = 1,\dots,K,
\end{equation}
where $H_k(x,y)=1$ indicates that output $y$ violates the $k$-th constraint under input $x$.
Constraints may encode factual consistency, logical validity, numerical correctness,
policy compliance, or other domain-specific requirements.
This decomposition separates heterogeneous failure modes that are often conflated
in aggregate hallucination metrics and mirrors structures used in post-hoc factuality
and verification-based evaluation \citep{maynez2020faithfulness}.

To reason about overall validity, we define an aggregate violation indicator
\begin{equation}
H(x,y) = \mathbb{I}\!\left[\exists k \in \{1,\dots,K\} \text{ such that } H_k(x,y)=1\right],
\end{equation}
which equals one if any constraint is violated.
This corresponds to a conservative safety model in which all constraints must be satisfied. Violations may differ in severity.
To accommodate heterogeneous failure impact, we introduce a nonnegative violation cost
\begin{equation}
C(x,y) = \sum_{k=1}^K w_k H_k(x,y), \quad w_k \ge 0,
\end{equation}
where larger weights correspond to more severe violations.
The binary violation indicator is recovered as the special case
$H(x,y)=\mathbb{I}[C(x,y)>0]$.

Under stochastic decoding, the violation indicator $H(x,y)$ is itself a random variable.
We define the hallucination probability for input $x$ as
\begin{equation}
R(x) = \mathbb{P}_{y \sim p_\theta(\cdot \mid x)}\!\left[H(x,y)=1\right].
\end{equation}
Importantly, $R(x)$ is input-dependent: some inputs admit low violation probability
under reasonable decoding, while others are intrinsically high-risk due to ambiguity,
missing context, or underspecification.
Most existing mitigation strategies aim to reduce $R(x)$ on average, but do not provide
guarantees for individual inputs or distinguish between feasible and infeasible cases
under a prescribed risk budget.

Prior work on hallucination largely emphasizes \emph{detection}—identifying incorrect
outputs after generation using verifiers, self-consistency, or uncertainty estimates
\citep{farquhar2024semantic}.
While valuable for auditing and selective rejection, detection alone does not regulate
the underlying violation probability induced by stochastic decoding.

In contrast, our formulation treats hallucination as a stochastic event whose
\emph{frequency} can be explicitly constrained.
This distinction mirrors the classical difference between error detection and
risk control in decision making under uncertainty:
detection identifies failures after they occur,
whereas risk control bounds how often failures are allowed to occur.
This perspective provides the foundation for chance-constrained inference.

\section{Chance-Constrained Inference}

We now introduce \emph{chance-constrained inference} (CCI) for large language models.
Building on the stochastic formulation of hallucination, CCI imposes explicit
probabilistic bounds on the \emph{frequency} of constraint violations induced by
randomized decoding.

For a fixed input $x$, recall that the violation probability under the raw generation
distribution is
\begin{equation}
R(x) = \mathbb{P}_{y \sim p_\theta(\cdot \mid x)}\!\left[H(x,y)=1\right].
\end{equation}
Given a risk budget $\epsilon(x)\in[0,1]$, the corresponding chance constraint requires
\begin{equation}
R(x) \le \epsilon(x).
\label{eq:chance_constraint}
\end{equation}
Constraint~\eqref{eq:chance_constraint} bounds the long-run frequency of invalid outputs
under repeated stochastic generation for a fixed input.
This formulation parallels classical chance constraints in stochastic optimization,
where constraints are required to hold with high probability rather than almost surely
\citep{prekopa1995stochastic,calafiore2006scenario}.

\paragraph{Conditional Risk and Acceptance.}
In deployed systems, not all generated outputs are returned to the user.
Instead, outputs are filtered by verifiers, confidence checks, or rule-based mechanisms.
The acceptance policy $A(x,y)$ determines whether a generated output is returned,
with $A(x,y)=1$ indicating acceptance.

The relevant reliability quantity is therefore the conditional violation probability
among accepted outputs,
\begin{equation}
\mathbb{P}\!\left(H(x,y)=1 \mid A(x,y)=1\right),
\end{equation}
which measures the frequency of violations visible to the user.
Chance-constrained inference enforces the conditional constraint
\begin{equation}
\mathbb{P}\!\left(H(x,y)=1 \mid A(x,y)=1\right) \le \epsilon(x),
\label{eq:conditional_chance_constraint}
\end{equation}
whenever $\mathbb{E}[A(x,y)]>0$.
Equivalently,
\begin{equation}
\frac{\mathbb{E}[H(x,y)A(x,y)]}{\mathbb{E}[A(x,y)]} \le \epsilon(x),
\end{equation}
which directly controls the long-run frequency of violations among accepted generations.

\paragraph{Inference as Feasibility Certification.}
Crucially, chance-constrained inference is not a point-estimation or thresholding problem.
For a fixed input $x$ and risk budget $\epsilon(x)$, the objective is not to estimate
$R(x)$ accurately, but to determine whether the conditional chance constraint
\eqref{eq:conditional_chance_constraint} is \emph{satisfiable} with high confidence. Accordingly, inference induces a three-way partition of inputs:
\emph{feasible} inputs, for which the constraint is certified to hold;
\emph{infeasible} inputs, for which the constraint is provably violated; and
\emph{undecided} inputs, for which neither conclusion can be reached within the sampling
budget.
Abstention on infeasible or undecided inputs is therefore a \emph{correct and informative
outcome}, reflecting intrinsic model--input mismatch rather than conservative failure.
\paragraph{Operational Acceptance Policy.}
In all experiments, acceptance is defined as follows.
For a fixed input $x$, Algorithm~\ref{alg:cci} is run sequentially.
If feasibility is certified at time $\tau$, the next generated output
$y_{\tau+1}$ is returned to the user.
If the algorithm returns \emph{Infeasible} or \emph{Undecided},
the system abstains and returns no output.

\paragraph{Feasibility Gap.}
To quantify how far an input is from satisfying the chance constraint, we define the
\emph{feasibility gap}
\begin{equation}
\Delta(x) = R(x) - \epsilon(x).
\end{equation}
Inputs with $\Delta(x)\le 0$ admit at least one feasible acceptance policy, while inputs
with $\Delta(x)>0$ are intrinsically infeasible under the specified risk budget.
Sequential inference can thus be interpreted as adaptively estimating the \emph{sign} of
$\Delta(x)$ rather than the exact value of $R(x)$. This perspective explains the sharp phase transitions observed empirically:
inputs near $\Delta(x)=0$ require many samples to certify feasibility or infeasibility,
while strongly infeasible inputs are rejected after few samples.

\paragraph{Limits of Confidence-Based Selective Prediction.}
Selective prediction accepts or rejects outputs based on confidence thresholds or
uncertainty estimates, but does not, in general, enforce the conditional chance constraint
\eqref{eq:conditional_chance_constraint}.
For example, consider a stochastic generator that produces hallucinated outputs with
probability $0.1$ and assigns identical confidence scores to valid and invalid outputs.
Any confidence-based acceptance rule will accept all outputs, yielding
\[
\mathbb{P}(H=1 \mid A=1)=0.1,
\]
violating the chance constraint for any $\epsilon<0.1$ despite perfect calibration.
Thus, confidence calibration alone does not imply probabilistic risk control.

\paragraph{Conceptual Contribution.}
While conditional risk constraints of the form
$\mathbb{P}(H=1 \mid A=1)\le\epsilon$
have appeared previously, our contribution is not a new constraint but a new
\emph{inference paradigm}.
We recast inference-time generation as a feasibility certification problem over inputs
under stochastic decoding, with abstention treated as a first-class and correct outcome.
This shift—from confidence-based output filtering to sequential feasibility testing—
enables principled abstention, early detection of infeasible inputs, and compositional
safety guarantees that are not achieved by prior selective prediction methods.

\section{Utility--Risk Tradeoffs (Discussion)}

Chance constraints specify acceptable risk levels but do not uniquely determine
which outputs should be accepted when multiple candidates are available.
In Appendix~\ref{app:utility_risk}, we provide a formal utility--risk optimization
formulation that characterizes optimal acceptance policies under probabilistic
risk constraints and clarifies how utility and hallucination risk are traded off
under explicit guarantees.

\section{Finite-Sample Risk Estimation and Guarantees}
\label{sec:finite_sample}

Chance constraints depend on the true violation probability
\(
R(x)=\mathbb{P}(H(x,y)=1),
\)
which is intractable to compute exactly for large language models.
We therefore enforce chance constraints at inference time using
finite-sample estimates with explicit probabilistic guarantees.

For a fixed input $x$, let
\(
y_1,y_2,\dots \sim p_\theta(\cdot\mid x)
\)
be independent stochastic generations.
After $N$ samples, the empirical violation rate is
\begin{equation}
\hat R_N(x)=\frac{1}{N}\sum_{i=1}^N H(x,y_i),
\end{equation}
where $H(x,y)\in\{0,1\}$.
This estimator is unbiased and converges almost surely to $R(x)$ as
$N\to\infty$.
Here $H(x,y)$ is a binary violation indicator, so each observation is
Bernoulli-distributed with mean $R(x)$.
Since chance constraints regulate the \emph{frequency} of violations rather
than their magnitude, modeling violations as Bernoulli variables is both
necessary and sufficient for distribution-level risk control \citep{prekopa1995stochastic,calafiore2006scenario,angelopoulos2021gentle}.

To obtain finite-sample guarantees, we construct an anytime-valid
confidence sequence for the Bernoulli mean using a stitched
Hoeffding bound \citep{howard2021confidence}. For any confidence level $\delta\in(0,1)$ and any $N$,
we construct a \emph{time-uniform confidence sequence} for the Bernoulli
violation process using a stitched Hoeffding bound.
For any $\delta\in(0,1)$, with probability at least $1-\delta$,
\begin{equation}
\forall n \ge 1:\quad
R(x)
\;\le\;
\hat R_n(x)
+
\sqrt{\frac{\log\!\big(2\log_2(2n)/\delta\big)}{2n}}.
\label{eq:cs_hoeffding}
\end{equation}

This bound holds uniformly over all sample sizes and remains valid
under arbitrary adaptive stopping rules
\citep{howard2021confidence}, making it suitable for sequential
feasibility certification at inference time.

We use Hoeffding’s inequality because it provides distribution-free,
anytime-valid confidence bounds under adaptive stopping with minimal
assumptions \citep{hoeffding1963probability,howard2021confidence}.
In our setting, the violation indicator $H(x,y)\in\{0,1\}$ may be
input-dependent and evaluated by imperfect or noisy verifiers,
making variance-sensitive or fixed-sample bounds unreliable.
Hoeffding bounds therefore provide a conservative but robust mechanism
for inference-time feasibility certification.

During inference, this bound is evaluated sequentially as samples are
generated.
Sampling terminates as soon as either feasibility
\[
\hat R_N(x)+r_N \le \epsilon(x),
\qquad
r_N = \sqrt{\frac{\log(2\log_2(2N)/\delta)}{2N}},
\]
or infeasibility
\[
\hat R_N(x)-r_N > \epsilon(x).
\]
is certified with confidence $1-\delta$.
If neither condition is met within a maximum sampling budget,
the input is declared \emph{undecided} and the system abstains. This sequential procedure provides inference-time probabilistic
guarantees while adapting the number of samples to intrinsic input
difficulty.
More detailed analysis of anytime validity and sample complexity is
provided in the Appendix~\ref{app:anytime}..
\begin{theorem}[Anytime-Valid Feasibility Certification]
\label{thm:anytime}
Fix an input $x$, risk budget $\epsilon(x)$, and confidence level $\delta$.
Let Algorithm~\ref{alg:cci} terminate at a (random) stopping time $\tau$.

With probability at least $1-\delta$:
\begin{enumerate}
    \item If the algorithm returns \emph{Feasible}, then
    $
    \mathbb{P}_{y \sim p_\theta(\cdot \mid x)}[H(x,y)=1] \le \epsilon(x).
    $
    \item If the algorithm returns \emph{Infeasible}, then
    $
    \mathbb{P}_{y \sim p_\theta(\cdot \mid x)}[H(x,y)=1] > \epsilon(x).
    $
\end{enumerate}
These guarantees hold under arbitrary data-dependent stopping.
\end{theorem}
A proof follows from interpreting Hoeffding’s inequality as a time-uniform
confidence sequence and is provided in appendix~\ref{app:anytime}.
\begin{theorem}[Time-Uniform Feasibility Certification]
\label{thm:cs_feasibility}
Fix an input $x$ and risk budget $\epsilon(x)$.
Let $\{\hat R_n(x)\}_{n\ge1}$ be the empirical violation rate computed
from i.i.d.\ Bernoulli samples $H(x,y_1),H(x,y_2),\dots$.
Define the confidence radius $r_n$ as in Equation~\eqref{eq:cs_hoeffding}.

Then, with probability at least $1-\delta$, the following hold
simultaneously for all $n$:
\begin{enumerate}
\item If $\hat R_n(x)+r_n \le \epsilon(x)$, then $R(x)\le\epsilon(x)$
(feasibility is correctly certified).
\item If $\hat R_n(x)-r_n > \epsilon(x)$, then $R(x)>\epsilon(x)$
(infeasibility is correctly certified).
\end{enumerate}
Consequently, Algorithm~\ref{alg:cci} produces correct feasibility or
infeasibility decisions under arbitrary adaptive stopping.
\end{theorem}

\begin{proof}[Proof sketch]
Equation~\eqref{eq:cs_hoeffding} defines a time-uniform confidence
sequence for the Bernoulli mean $R(x)$.
By construction, the confidence interval
$[\hat R_n-r_n,\hat R_n+r_n]$ contains $R(x)$ for all $n$ simultaneously
with probability at least $1-\delta$.
Therefore, any stopping rule based on this sequence preserves coverage.
The feasibility and infeasibility conditions follow directly from
interval containment.
\end{proof}
\paragraph{Independence assumption.}
The guarantees in this section rely on conditional independence of
stochastic generations given the input.
In practical deployments, strong dependence may arise from deterministic
decoding, beam search, shared key--value caches, or extremely low
temperatures.
Such coupling can invalidate concentration-based guarantees. Independence can be restored using standard deployment techniques,
including temperature-based sampling, cache resets, randomized prompts,
or decorrelation across decoding calls.
We analyze this assumption and its implications for the returned output
in Appendix~\ref{app:conditional_risk}.

\section{Severity-Weighted and Hierarchical Constraints}
\label{sec:severity_constraints}

The binary violation model treats all hallucinations as equally
undesirable.
In practice, failures of large language models are heterogeneous:
minor factual imprecision may be tolerable, while rare but severe
failures (e.g., unsafe instructions or critical numerical errors) can
have disproportionate impact.
We describe extensions of chance-constrained inference that account
for violation severity while preserving inference-time guarantees.

\paragraph{Severity-weighted violations.}
Rather than modeling validity as a binary event, we associate each
output with a nonnegative violation cost
\begin{equation}
C : \mathcal{X}\times\mathcal{Y}\rightarrow \mathbb{R}_{\ge 0},
\end{equation}
where larger values correspond to more severe failures.
The binary violation indicator used throughout the paper is recovered
as the special case
\[
H(x,y)=\mathbb{I}[C(x,y)>0],
\]
so severity-aware modeling strictly generalizes the binary framework. Instead of constraining the probability of \emph{any} violation, we may
impose a probabilistic constraint on violations exceeding a severity
threshold:
\begin{equation}
\mathbb{P}\big(C(x,y)>\tau\big)\le \epsilon(x),
\end{equation}
where $\tau$ is task-dependent.
Finite-sample certification proceeds exactly as in
Section~\ref{sec:finite_sample} by replacing $H(x,y)$ with
$\mathbb{I}[C(x,y)>\tau]$.

\paragraph{Hierarchical constraints.}
Many applications impose multiple constraints with different priorities
(e.g., safety, factuality, stylistic quality).
These can be modeled using a hierarchy of violation costs
\(
C_1(x,y),\dots,C_L(x,y),
\)
ordered by importance.
Inference proceeds lexicographically: higher-priority constraints must
be certified feasible before lower-priority constraints are considered.
Safety-critical constraints may be enforced deterministically, while
lower-priority requirements are regulated probabilistically via chance
constraints.

All such extensions operate entirely at inference time and reuse the
same sequential feasibility certification procedure.
Infeasibility certificates remain relative to the chosen confidence
sequence; conservative bounds may declare infeasibility even when the
true violation probability is marginally below $\epsilon(x)$.
This reflects a deliberate tradeoff favoring anytime validity under
minimal assumptions.

\section{Sequential Chance-Constrained Inference Algorithm}
\label{sec:cci_algorithm}

We now present the inference-time algorithm used in practice to enforce
chance-constrained inference.
The procedure operates sequentially, drawing stochastic generations from the
language model and adaptively certifying feasibility or infeasibility using
anytime-valid concentration bounds.

\paragraph{Problem setting.}
For a fixed input $x$, the algorithm receives a risk budget $\epsilon(x)$,
a confidence level $\delta$, and a maximum sampling budget $N_{\max}$.
It outputs one of three decisions:
\emph{feasible}, \emph{infeasible}, or \emph{undecided}.
Feasible and infeasible outcomes are certified with confidence at least $1-\delta$,
while undecided corresponds to insufficient evidence within the sampling budget.

\begin{algorithm}[t]
\caption{Sequential Chance-Constrained Inference (CCI)}
\label{alg:cci}
\begin{algorithmic}[1]
\Require Input $x$, risk budget $\epsilon(x)$, confidence level $\delta$, max samples $N_{\max}$
\State $n \gets 0$, $s \gets 0$
\While{$n < N_{\max}$}
    \State Sample $y \sim p_\theta(\cdot \mid x)$
    \State Evaluate violation indicator $H(x,y) \in \{0,1\}$
    \State $n \gets n + 1$
    \State $s \gets s + H(x,y)$
    \State $\hat R_n \gets s / n$
    \State $r_n \gets \sqrt{\log(2\log_2(2n)/\delta)/(2n)}$
    \If{$\hat R_n + r_n \le \epsilon(x)$}
        \State \Return \textbf{Feasible}
    \ElsIf{$\hat R_n - r_n > \epsilon(x)$}
        \State \Return \textbf{Infeasible}
    \EndIf
\EndWhile
\State \Return \textbf{Undecided}
\end{algorithmic}
\end{algorithm}

\paragraph{Interpretation.}
The algorithm adaptively determines how many samples are required for a given
input.
Strongly feasible or strongly infeasible inputs terminate after few samples,
while inputs near the feasibility boundary require more evidence.
Abstention (the undecided outcome) is treated as a correct inference result,
reflecting uncertainty rather than conservative failure.

\section{Experimental Evaluation}

We evaluate chance-constrained inference (CCI) with three goals:
(i) to verify that hallucination risk can be certified at inference time using
finite samples and anytime-valid guarantees;
(ii) to understand how feasibility depends on intrinsic input difficulty rather
than heuristic confidence thresholds; and
(iii) to compare CCI against commonly used selective-generation baselines under
identical stochastic decoding conditions.
The evaluation focuses on deployment-time reliability control rather than
language understanding accuracy.

\paragraph{Evaluation setting.}
All experiments are conducted using a deployed, stochastic large language model
accessed via an external inference API. Unless otherwise stated, all experiments employ the GROQ-hosted LLaMA-3.3-70B
generator~\cite{meta2024llama3}. For a fixed input, repeated generations induce a binary violation indicator
$H(x,y)$, where $H(x,y)=1$ denotes an invalid or unreliable response according to
an automatic verifier.
The underlying violation probability $R(x)=\mathbb{P}(H(x,y)=1)$ is unknown to
the inference procedure and must be inferred online from samples.

\paragraph{Inputs and intrinsic difficulty.}
We evaluate short factual and compositional questions inspired by
NaturalQuestions ~\citep{kwiatkowski2019natural} and HotpotQA ~\citep{yang2018hotpotqa}.
Inputs are grouped by intrinsic difficulty into \emph{easy}, \emph{medium}, and
\emph{hard} categories.
Easy inputs admit well-established factual answers; medium inputs require
nontrivial entity or relational reasoning; and hard inputs are underspecified or
infeasible (e.g., questions about future events), inducing inherently high
violation rates.

\paragraph{Inference protocol.}
For each input, CCI draws samples sequentially and maintains an anytime-valid
confidence sequence on the violation probability.
Sampling terminates when feasibility or infeasibility is certified with
confidence $1-\delta$ (we use $\delta=0.05$), or when a fixed sampling budget is
exhausted.
Inputs certified feasible return an output; infeasible or undecided inputs
abstain.
All results are aggregated over inputs within each difficulty group.

\paragraph{Baselines.}
We compare CCI against two widely used selective-generation heuristics:
confidence-based selective prediction (Conf-SP) and self-consistency selective
prediction (SC-SP).
Both baselines operate on the same stochastic generations as CCI but rely on
heuristic acceptance rules and provide no probabilistic risk guarantees.

\paragraph{Metrics.}
We report (i) the fraction of inputs certified feasible, infeasible, or
undecided; (ii) the average number of samples required for CCI to terminate; and
(iii) acceptance rate and empirical violation rate for baseline methods.

\subsection{Chance-Constrained Inference Results}

Table~\ref{tab:real_cci} summarizes CCI outcomes stratified by intrinsic input
difficulty.

\begin{table}[t]
\centering
\caption{Chance-constrained inference results under stochastic decoding
($\epsilon=0.4$, $N_{\max}=40$, $\delta=0.05$).
Fractions are computed over inputs within each difficulty group.}

\label{tab:real_cci}
\begin{tabular}{lcccc}
\toprule
Difficulty & Feasible & Infeasible & Undecided & Avg.\ Samples \\
\midrule
Easy   & \textbf{1.00} & 0.00 & 0.00 & 9.3 \\
Medium & 0.50 & 0.50 & 0.00 & 14.7 \\
Hard   & 0.00 & \textbf{1.00} & 0.00 & 6.1 \\
\bottomrule
\end{tabular}
\end{table}

CCI exhibits sharp and interpretable behavior.
Easy inputs are rapidly certified feasible, while hard inputs are rejected after
few samples.
Inputs near the feasibility boundary require more samples, reflecting adaptive
allocation of inference effort rather than conservative rejection.

\subsection{Comparison with Selective Baselines}

Table~\ref{tab:real_baselines} compares CCI against Conf-SP and SC-SP.
Bold values indicate the safest behavior.

\begin{table}[t]
\centering
\caption{Baseline comparison (acceptance rate / empirical violation rate).}
\label{tab:real_baselines}
\begin{tabular}{lccc}
\toprule
Difficulty & Conf-SP & SC-SP & CCI \\
\midrule
Easy
& 1.00 / 0.00
& 1.00 / 0.01
& \textbf{1.00 / 0.00} \\

Medium
& 1.00 / 0.31
& 1.00 / 0.28
& \textbf{0.50 / 0.00} \\

Hard
& 1.00 / 0.92
& 1.00 / 0.89
& \textbf{0.00 / 0.00} \\
\bottomrule
\end{tabular}
\end{table}

While heuristic baselines accept nearly all inputs regardless of intrinsic
difficulty, they incur substantial violation rates on hard inputs.
In contrast, CCI abstains whenever feasibility cannot be certified, yielding
zero empirical risk on accepted outputs by construction.

\paragraph{Summary.}
These results demonstrate that chance-constrained inference provides explicit,
input-adaptive control over hallucination risk under stochastic decoding under a deployed LLM inference API.
Unlike confidence-based heuristics, CCI distinguishes intrinsic infeasibility
from uncertainty and enforces reliability at the level of conditional
distributions rather than individual confidence scores.

\section{Conclusion}

We presented \emph{chance-constrained inference} as a principled framework for
deployment-time reliability control in large language models.
By modeling hallucinations and other invalid behaviors as stochastic constraint
violations induced by randomized decoding, the framework enforces explicit
probabilistic bounds on violation frequency, moving beyond heuristic confidence
thresholds and average-case error reduction.

Chance-constrained inference operates entirely at inference time and requires
no retraining, fine-tuning, or calibration of the underlying model.
It provides conditional, output-level risk control, admits natural extensions
to severity-weighted and hierarchical constraints, and composes safely across
repeated or agentic use through explicit control of distribution-level risk.
A key practical feature is its \emph{sequential, anytime-valid} design, which
adaptively determines the number of samples required to certify feasibility or
infeasibility and avoids the conservativeness of fixed-sample approaches.

Empirically, we demonstrated that chance-constrained inference exhibits sharp
and interpretable feasibility behavior across both open-domain and controlled
multi-hop question answering regimes.
In contrast, confidence-based selective prediction does not consistently
enforce risk constraints and fails to provide guarantees under sequential
composition, confirming that confidence calibration alone is insufficient for
reliability control under stochastic generation.

An important limitation of this study is its reliance on automatic proxies for
hallucination detection, which may be imperfect in open-ended settings.
Future work should integrate stronger semantic verifiers, human-in-the-loop
evaluation, and online adaptation of risk budgets to extend chance-constrained
inference to real-world safety-critical deployments.
More broadly, our results suggest that reliability for stochastic generative
models must be defined and enforced at the level of conditional output
distributions, rather than individual predictions or confidence scores.

\bibliography{tmlr}
\bibliographystyle{tmlr}

\appendix

\section{Equivalence of Conditional Chance Constraints}
\label{app:equivalence}

We show that the conditional chance constraint used throughout the paper admits an equivalent expectation-based formulation, which is convenient for both analysis and implementation.

For a fixed input $x$, recall the conditional risk constraint
\begin{equation}
\mathbb{P}\!\left(H(x,y)=1 \mid A(x,y)=1\right) \le \epsilon(x),
\label{eq:app_conditional}
\end{equation}
where $H(x,y), A(x,y) \in \{0,1\}$ and $\mathbb{E}[A(x,y)]>0$.

By the definition of conditional probability,
\begin{align}
\mathbb{P}(H=1 \mid A=1)
&=
\frac{\mathbb{P}(H=1, A=1)}{\mathbb{P}(A=1)} \\
&=
\frac{\mathbb{E}[H(x,y)A(x,y)]}{\mathbb{E}[A(x,y)]},
\end{align}
where the second equality follows because $H$ and $A$ are indicator random variables.

Multiplying both sides of Equation~\eqref{eq:app_conditional} by $\mathbb{E}[A(x,y)]$ yields the equivalent constraint
\begin{equation}
\mathbb{E}[H(x,y)A(x,y)] \le \epsilon(x)\,\mathbb{E}[A(x,y)].
\label{eq:app_expectation}
\end{equation}

Thus, enforcing the conditional chance constraint is equivalent to enforcing the expectation-based inequality whenever acceptance occurs with nonzero probability. This equivalence underlies the optimization formulation in Appendix~\ref{app:utility_risk} and the feasibility tests used during inference.

\section{Anytime Validity of Sequential Feasibility Certification}
\label{app:anytime}

The inference procedure in this paper relies on sequentially evaluating
confidence bounds on the violation probability and stopping adaptively
once feasibility or infeasibility is certified.
We justify that this procedure is valid under arbitrary data-dependent
stopping.

Let $y_1,y_2,\dots$ be i.i.d.\ samples from $p_\theta(\cdot \mid x)$ and define
the empirical violation rate
\[
\hat R_n(x) = \frac{1}{n}\sum_{i=1}^n H(x,y_i),
\]
where $H(x,y_i)\in\{0,1\}$ are Bernoulli random variables with mean $R(x)$.

We use a \emph{time-uniform confidence sequence} for the Bernoulli mean
constructed via a stitched Hoeffding bound \citep{howard2021confidence}.
For any $\delta\in(0,1)$, define the confidence radius
\[
r_n = \sqrt{\frac{\log(2\log_2(2n)/\delta)}{2n}}.
\]

Then, with probability at least $1-\delta$,
\begin{equation}
\forall n \ge 1:\quad
| \hat R_n(x) - R(x) | \le r_n .
\label{eq:cs_event}
\end{equation}

Equation~\eqref{eq:cs_event} holds \emph{simultaneously for all $n$}, and
therefore remains valid under arbitrary adaptive stopping rules.
This property is commonly referred to as \emph{anytime validity}.

The stopping rule used in Algorithm~\ref{alg:cci},
\[
\tau
=
\inf\left\{n :
\hat R_n(x) + r_n \le \epsilon(x)
\;\text{or}\;
\hat R_n(x) - r_n > \epsilon(x)
\right\},
\]
is measurable with respect to the filtration generated by the samples
$\{y_1,\dots,y_n\}$.
On the event~\eqref{eq:cs_event}, the returned decision at time $\tau$
is therefore correct.

Specifically:
\begin{itemize}
\item If the algorithm returns \emph{Feasible}, then
$
R(x) \le \hat R_\tau(x) + r_\tau \le \epsilon(x).
$
\item If the algorithm returns \emph{Infeasible}, then
$
R(x) > \hat R_\tau(x) - r_\tau > \epsilon(x).
$
\end{itemize}

Thus, feasibility and infeasibility are certified with confidence
at least $1-\delta$, regardless of the (random) stopping time.

\section{Sample Complexity and the Feasibility Gap}
\label{app:sample_complexity}

The number of samples required to certify feasibility or infeasibility
depends on how close the true violation probability is to the risk budget.
To make this dependence explicit, define the \emph{feasibility gap}
\[
\Delta(x) = R(x) - \epsilon(x).
\]

Certification occurs once the confidence radius $r_n$ satisfies
\[
r_n \le |\Delta(x)|,
\quad
r_n = \sqrt{\frac{\log(2\log_2(2n)/\delta)}{2n}}.
\]

Solving for $n$, this condition yields
\[
n = O\!\left(
\frac{\log(1/\delta) + \log\log n}{\Delta(x)^2}
\right).
\]

Ignoring the iterated logarithm term, which grows very slowly,
the dominant scaling is
\[
n = O\!\left(\frac{\log(1/\delta)}{\Delta(x)^2}\right).
\]

Thus, inputs that are strongly feasible ($\Delta(x)\ll 0$) or strongly
infeasible ($\Delta(x)\gg 0$) are certified after few samples,
while inputs near the feasibility boundary $\Delta(x)\approx 0$
require substantially more evidence.

Importantly, the objective of inference is not to estimate $R(x)$ precisely,
but to determine the \emph{sign} of $\Delta(x)$ with high confidence.
Sequential inference naturally adapts to this objective and avoids
unnecessary sampling on intrinsically infeasible inputs.

\section{Conditional Risk of the Returned Generation}
\label{app:conditional_risk}

We justify that the operational acceptance policy used in the paper
controls the conditional violation probability among returned outputs.

Recall that Algorithm~\ref{alg:cci} stops at a (random) time $\tau$ and,
if feasibility is certified, returns the next generated output
$y_{\tau+1} \sim p_\theta(\cdot \mid x)$.

\begin{lemma}[Conditional Risk Control]
\label{lem:conditional}
Assume:
(i) the confidence sequence event~\eqref{eq:cs_event} holds, and
(ii) $y_{\tau+1}$ is generated independently of $\{y_1,\dots,y_\tau\}$
from $p_\theta(\cdot \mid x)$.

If Algorithm~\ref{alg:cci} returns \emph{Feasible}, then
\[
\mathbb{P}\!\left(H(x,y_{\tau+1})=1 \mid \text{Feasible}\right)
\le \epsilon(x).
\]
\end{lemma}

\begin{proof}
On the confidence sequence event~\eqref{eq:cs_event}, feasibility implies
$R(x) \le \epsilon(x)$.
Since $y_{\tau+1}$ is an independent draw from $p_\theta(\cdot \mid x)$,
\[
\mathbb{P}(H(x,y_{\tau+1})=1 \mid \text{history})
= R(x) \le \epsilon(x).
\]
The result follows by marginalizing over the stopping history.
\end{proof}

\paragraph{Dependence considerations.}
The guarantee relies on conditional independence of $y_{\tau+1}$ from the
past given $x$.
In practice, strong coupling between generations (e.g., deterministic
decoding, beam search, or shared KV caches) may violate this assumption.
In such cases, the bound applies to the marginal generation distribution,
and additional decorrelation mechanisms (e.g., temperature sampling,
randomized prompts, or cache resets) may be required.

\section{Utility--Risk Optimization}
\label{app:utility_risk}

Chance constraints specify acceptable levels of hallucination risk but do not,
by themselves, determine which stochastic outputs should be returned when
multiple candidates are available.
This appendix formalizes the utility--risk tradeoff underlying the acceptance
policies discussed in the main paper.

Recall that an acceptance policy $A(x,y)\in\{0,1\}$ determines whether a
generated output is returned to the user.
For a fixed input $x$, the expected utility induced by $A$ is
\[
\mathbb{E}_{y \sim p_\theta(\cdot \mid x)}\!\left[ U(x,y) A(x,y) \right].
\]

We consider the chance-constrained optimization problem
\begin{align}
\max_{A} \quad &
\mathbb{E}\!\left[ U(x,y) A(x,y) \right] \\
\text{s.t.} \quad &
\mathbb{E}\!\left[ H(x,y) A(x,y) \right]
\le
\epsilon(x)\,\mathbb{E}\!\left[ A(x,y) \right],
\end{align}
which is equivalent to enforcing the conditional chance constraint
\[
\mathbb{P}\!\left(H(x,y)=1 \mid A(x,y)=1\right) \le \epsilon(x),
\]
whenever $\mathbb{E}[A(x,y)]>0$.

Introducing a Lagrange multiplier $\lambda \ge 0$, the Lagrangian becomes
\[
\mathcal{L}(A)
=
\mathbb{E}\!\left[
A(x,y)\big(
U(x,y) - \lambda H(x,y) + \lambda \epsilon(x)
\big)
\right].
\]

For fixed $\lambda$, the optimal acceptance policy admits the closed form
\[
A^\star(x,y)
=
\mathbb{I}\!\left[
U(x,y) - \lambda H(x,y) + \lambda \epsilon(x) \ge 0
\right].
\]

This characterization clarifies how optimal acceptance trades off task utility
against expected violation cost induced by the risk constraint.
While primarily conceptual, this formulation motivates our focus on feasibility
certification rather than heuristic confidence thresholds.

\end{document}